%% file: foo.tex
\documentclass{article}
\input{gmacros}  
\usepackage[numbers]{natbib}

\usepackage[preprint]{neurips_2021}

\usepackage[utf8]{inputenc} 
\usepackage[T1]{fontenc}    
\usepackage{hyperref}       
\usepackage{url}            
\usepackage{booktabs}       
\usepackage{amsfonts}       
\usepackage{nicefrac}       
\usepackage{microtype}      
\usepackage{xcolor}         

\usepackage{microtype}
\usepackage{graphicx}
\usepackage{subfigure}
\usepackage{booktabs} 
\usepackage{amssymb,bm}
\usepackage{amsmath}

\usepackage{natbib}

\usepackage{xspace}
\usepackage{comment}
\usepackage{selectp}
\usepackage{amsthm}
\usepackage{algorithm}

\newtheorem{theorem}{Theorem} 
\newtheorem{lemma}{Lemma} 
\newtheorem{proposition}{Proposition} 
\newtheorem{assumption}{Assumption}
\usepackage{times}

\usepackage{times}

\title{Online Learning in MDPs with \\ Linear Function Approximation and Bandit Feedback}

\author{%
  Gergely Neu \\
  Universitat Pompeu Fabra\\ 
  Barcelona, Spain\\
  \texttt{gergely.neu@gmail.com} \\
   \And
   Julia Olkhovskaya\\
     Universitat Pompeu Fabra\\ 
     Barcelona, Spain\\
     \texttt{julia.olkhovskaya@gmail.com}
}

\begin{document}

\maketitle

\begin{abstract}
 We consider the problem of online learning in an episodic Markov decision process, where the reward function is 
allowed to change between episodes in an adversarial manner and the learner only observes the rewards associated with 
its actions. We assume that rewards and the transition function can be represented as linear functions in terms of a 
known low-dimensional feature map, which allows us to consider the setting where  the state space is arbitrarily large. 
We also assume that the learner has a perfect knowledge of the MDP dynamics.
Our main contribution is developing an algorithm whose expected regret after
$T$ episodes is bounded by $\widetilde{\mathcal{O}}\bpa{\sqrt{dHT}}$, where $H$ is the number of steps in each episode 
and $d$ is the dimensionality of the feature map.
\end{abstract}

\input{intro}

\input{definition}

\input{results}

\input{analysis}

\input{discussion}

\bibliographystyle{abbrvnat}
\bibliography{mdp,shortconfs}

\newpage

\appendix

\newpage

\input{appendixA}

\input{appendixB}

\input{appendixD}
\input{appendixF}

\end{document}

%% file: gmacros.tex
 \newcommand{\DC}{D_C}
 
 \newcommand{\GG}{\mathcal{G}}

 \newcommand{\X}{\mathcal{X}}
 \newcommand{\U}{\mathcal{U}}

\newcommand{\intg}{\mathbb{Z}}
\newcommand{\real}{\mathbb{R}}

\newcommand{\OO}{\mathcal{O}}

\newcommand{\trace}[1]{\mbox{tr}\left(#1\right)}

\newcommand{\EE}[1]{\mathbb{E}\left[#1\right]}

\newcommand{\EEtb}[1]{\mathbb{E}_t\bigl[#1\bigr]}

\newcommand{\EEs}[2]{\mathbb{E}_{#2}\left[#1\right]}

\newcommand{\PPpi}[1]{\mathbb{P}_{\pi}\left[#1\right]}
\newcommand{\EEt}[1]{\mathbb{E}_t\left[#1\right]}

\newcommand{\EEcc}[2]{\mathbb{E}\left[\left.#1\right|#2\right]}

\def\argmin{\mathop{\mbox{ arg\,min}}}

\newcommand{\ra}{\rightarrow}

\newcommand{\siprod}[2]{\langle#1,#2\rangle}
\newcommand{\iprod}[2]{\left\langle#1,#2\right\rangle}
\newcommand{\biprod}[2]{\bigl\langle#1,#2\bigr\rangle}

\newcommand{\norm}[1]{\left\|#1\right\|}

\newcommand{\onenorm}[1]{\norm{#1}_1}
\newcommand{\twonorm}[1]{\norm{#1}_2}
\newcommand{\infnorm}[1]{\norm{#1}_\infty}
\newcommand{\opnorm}[1]{\norm{#1}_{\text{op}}}

\newcommand{\ev}[1]{\left\{#1\right\}}
\newcommand{\pa}[1]{\left(#1\right)}
\newcommand{\bpa}[1]{\bigl(#1\bigr)}

\newcommand{\wh}{\widehat}
\newcommand{\wt}{\widetilde}

\newcommand{\loss}{r}

\newcommand{\tmu}{\wt{\mu}}

\newcommand{\ssp}{\mathcal{X}}
\newcommand{\actions}{\mathcal{A}}

\newcommand{\A}{\mathcal{A}}

\newcommand{\lambdamin}{\lambda_{\min}}

\newcommand{\hmu}{\wh{\mu}}
\newcommand{\hpi}{\wh{\pi}}
\newcommand{\hu}{\wh{u}}

\newcommand{\hr}{\wh{r}}

\newcommand{\htheta}{\wh{\theta}}

\newcommand{\ttheta}{\wt{\theta}}

\newcommand{\bW}{W}

\newcommand{\Sp}{\Sigma^+}
\newcommand{\hSp}{\wh{\Sigma}^+}

\newcommand{\transpose}{^\mathsf{\scriptscriptstyle T}}

\newcommand{\bone}{\bm{1}}

\usepackage{todonotes}
\definecolor{PalePurp}{rgb}{0.66,0.57,0.66}

\newcommand{\regret}{\mathfrak{R}}
\newcommand{\hregret}{\widehat{\mathfrak{R}}}

\newcommand{\linexprl}{\textsc{MDP-LinExp3}\xspace}
\newcommand{\OQREPS}{\textsc{Online Q-REPS}\xspace}
\newcommand{\QREPS}{\textsc{Q-REPS}\xspace}
\newcommand{\OREPS}{\textsc{O-REPS}\xspace}
\newcommand{\ftrlrl}{\OQREPS}

\newcommand{\jmlrQED}{$\blacksquare$}

%% file: intro.tex
\section{Introduction}
We study the problem of online learning in episodic Markov Decision Processes (MDP), modeling a sequential 
decision making problem where the interaction between a learner and its environment is divided into $T$ episodes of  
fixed length $H$. At each time step of the episode, the learner observes the current state of the environment, chooses 
one of the available actions, and earns a reward. Consequently, the state of the environment changes according to the 
transition function of the underlying MDP, as a function of the previous state and the action taken by the learner. 
A key distinguishing feature of our setting is that we assume that the reward function can change arbitrarily between 
episodes, and the learner only has access to bandit feedback: instead of being able to observe the reward function at 
the end of the episode, the learner only gets to observe the rewards that it actually received. As traditional in this 
line of work, we aim to design algorithms for the learner with theoretical guarantees on her regret, which is the 
difference between the total reward accumulated by the learner and the total reward of the best stationary policy fixed 
in hindsight.

Unlike most previous work on this problem, we allow the state space to be very large and aim to prove performance 
guarantees that do not depend on the size of the state space, bringing theory one step closer to practical scenarios 
where assuming finite state spaces is unrealistic. To address the challenge of learning in large state spaces, we adopt 
the classic RL technique of using \emph{linear function approximation} and suppose that we have access to a relatively 
low-dimensional feature map that can be used to represent policies and value functions. We will assume that the feature 
map is expressive enough so that all action-value functions can be expressed as linear functions of the features, and 
that the learner has full knowledge of the transition function of the MDP.

Our main contribution is designing a computationally efficient algorithm called \OQREPS, and prove that in the 
setting described above, its regret is at most $\mathcal{O}\bpa{\sqrt{dHT D\pa{\mu^*\|\mu_0}}}$, where 
$d$ is the dimensionality of the feature map and $D\pa{\mu^*\|\mu_0}$ is the relative entropy between the state-action 
distribution $\mu^*$ induced by the optimal policy and an initial distribution $\mu_0$ given as input to the algorithm. 
Notably, our results do not require the likelihood ratio between these distributions to be uniformly bounded, and the 
bound shows no dependence on the eigenvalues of the feature covariance matrices. Our algorithm itself requires solving 
a $d^2$-dimensional convex optimization problem at the beginning of each episode, which can be solved to arbitrary 
precision $\varepsilon$ in time polynomial in $d$ and $1/\varepsilon$, independently of the size of the state-action 
space.

Our work fits into a long line of research considering online learning in Markov decision processes. The problem of 
regret minimization in stationary MDPs with a \emph{fixed} reward function has been studied extensively since the work 
of \citet{BK97,auer06logarithmic,tewari08optimistic,jaksch10ucrl}, with several important advances made in the past 
decade \citep{DB15,DLB17,AOM17,FPL18,JAZBJ18}. While most of these works considered small finite state spaces, the same 
techniques have been very recently extended to accommodate infinite state spaces under the assumption of 
realizable function approximation by \citet{2019arXiv190705388J} and \citet{YW19}. In particular, the notion of 
\emph{linear MDPs} introduced by \citet{2019arXiv190705388J}  has become a standard model for linear 
function approximation and has been used in several recent works (e.g., \citealp{NPB20, wei2020learning, 
agarwal2020flambe}).

Even more relevant is the line of work considering adversarial rewards, initiated by \citet{10.1287/moor.1090.0396}, 
who consider online learning in continuing MDPs with full feedback about the rewards. They proposed a MDP-E algorithm, 
that achieves $\OO(\tau^2{\sqrt{T  \log K}})$ regret, where $\tau$ is an upper bound on the mixing time of the MDP. 
Later, \citet{NGS13} proposed an algorithm which guarantees $\widetilde{\mathcal{O}}\bpa{\sqrt{\tau^3{KT/\alpha}}}$ 
regret with bandit feedback, essentially assuming that all states are reachable with probability  $\alpha > 0$ under all 
policies. In our work, we focus on episodic MDPs with a fixed episode length $H$. The setting was first considered in 
the bandit setting by \citet{NGS10}, who proposed an algorithm with a regret bound of $\OO(H^2\sqrt{T K }/\alpha)$. 
Although the number of states does not appear explicitly in the bound, the regret scales at least linearly with the 
size of the state space $\ssp$, since $|\ssp|\le H/ \alpha$. Later work by \citet{NIPS20134974,DGySz14} eliminated the 
dependence on $\alpha$ and proposed an algorithm achieving $\widetilde{\mathcal{O}}( \sqrt{TH|\ssp|K})$ regret. Regret 
bounds for the full-information case without prior knowledge of the MDP were achieved by \citet{pmlr-v22-neu12} and 
\citet{pmlrv97rosenberg19a}, of order $\widetilde{\mathcal{O}}( H|\ssp|K\sqrt{T})$ and $\widetilde{\mathcal{O}}( 
H|\ssp|\sqrt{KT})$, respectively. These results were recently extended to handle bandit feedback about the rewards by 
\citet{JJLSY20}, ultimately resulting in a regret bound of $\widetilde{\mathcal{O}}( H|\ssp|\sqrt{KT})$.

As apparent from the above discussion, all work on online learning in MDPs with adversarial rewards considers finite state spaces. 
The only exception we are aware of is the recent work  of \citet{2019arXiv191205830C}, whose algorithm  OPPO is 
guaranteed to achieve $\widetilde{\mathcal{O}}\bpa{\sqrt{  d^3H^3 T} }$, assuming that the learner has access to 
$d$-dimensional features that can perfectly represent all action-value functions. While \citet*{2019arXiv191205830C} 
remarkably assumed no prior knowledge of the MDP parameters, their guarantees are only achieved in the full-information 
case. This is to be contrasted with our results that are achieved for the much more restrictive bandit setting, albeit 
with the stronger assumption of having full knowledge of the underlying MDP, as required by virtually all prior work 
in the bandit setting, with the exception of \citet{JJLSY20}.

Our results are made possible by a careful combination of recently proposed techniques for contextual bandit 
problems and optimal control in Markov decision processes. In particular, a core component of our algorithm is a 
regularized linear programming formulation of optimal control in MDPs due to \citet{qreps}, which allows us to reduce 
the task of computing near-optimal policies in linear MDPs to a low-dimensional convex optimization problem. A similar 
algorithm design has been previously used for tabular MDPs by \citet{NIPS20134974,DGySz14}, with the purpose 
of removing factors of $1/\alpha$ from the previous state-of-the-art bounds of \citet{NGS10}. Analogously to this 
improvement, our methodology enables us to make strong assumptions on problem-dependent constants like likelihood ratios 
between $\mu^*$ and $\mu_0$ or eigenvalues of the feature covariance matrices. Another important building block of our 
method is a version of the recently proposed Matrix Geometric Resampling procedure of \citet{2020NO} that enables us to 
efficiently estimate the reward functions. Incorporating these estimators in the algorithmic template of \citet{qreps} 
is far from straightforward and requires several subtle adjustments.

\paragraph{Notation.}
We use $\iprod{\cdot}{\cdot}$ to denote inner products in Euclidean space and by $\norm{\cdot}$ we denote the Euclidean
norm for vectors and the operator norm for matrices. For a symmetric positive definite matrix $A$, we use 
$\lambdamin(A)$  to denote its smallest eigenvalue. We write $\trace{A}$ for the trace of a matrix $A$ and use $A 
\succcurlyeq 0$ to denote that an operator $A$ is positive semi-definite, and we use $A \succcurlyeq B $ to denote $A-B 
\succcurlyeq 0$. For a $d$-dimensional vector $v$, we denote the corresponding $d\times d$ diagonal matrix by 
$\mbox{diag}(v)$. For a positive integer $N$, we use $[N]$ to denote the set of positive integers $\ev{1,2,\dots,N}$. 
Finally, we will denote the set of all probability distributions over any set $\X$ by $\Delta_{\X}$.

%% file: definition.tex
\section{Preliminaries}\label{sec:problem_def}
An episodic Markovian Decision Process (MDP), denoted by $M = (\ssp, \actions, H, P, r)$ 
is defined by a state space $\ssp$, action space $\mathcal{A}$, episode length $H \in 
\intg_{+}$, transition function $P:\ssp\times\A\ra \Delta_{\ssp}$ and a reward function 
$\loss:\ssp \times \mathcal{A} \to [0,1]$.
For convenience, we will assume that both $\ssp$ and $\A$ are finite sets, although we allow the state space $\ssp$ to 
be arbitrarily large. Without significant loss of generality, we will assume that the set of available actions is the 
same $\A$ in each state, with cardinality $|\A| = K$. Furthermore, without any loss of generality, we will assume that 
the MDP has a layered structure, satisfying the following conditions:
\begin{itemize}
	\item The state set $\ssp$ can be decomposed into $H$ disjoint sets: $\ssp = 
\cup_{h =1}^H\ssp_h $,
	\item $\ssp_1 = \{x_1 \}$ and $\ssp_H = \{x_H \}$  are  singletons, 
	\item transitions are only possible between consecutive layers, that is, for any $x_h\in 
\ssp_h$, the distribution $P(\cdot|x,a)$ is supported on $\ssp_{h+1}$ for all $a$ and $h\in[H-1]$.
\end{itemize}
These assumptions are common in the related literature (e.g., \citealp{NGS10,NIPS20134974, 
pmlrv97rosenberg19a}) and are not essential to our analysis; their primary role is simplifying our 
notation.

In the present paper, we consider an \emph{online learning} problem where the learner interacts 
with its environment in a sequence of episodes $t=1,2,\dots,T$, facing a different \emph{reward 
functions} $r_{t,1}, \dots r_{t,H+1}$ selected by a (possibly adaptive) adversary at the beginning of each episode 
$t$. Oblivious to the reward function chosen by the adversary, the learner starts interacting with the 
MDP in each episode from the initial state $X_{t, 1}= x_1$. At each consecutive step $h \in [H-1]$ 
within the episode, the learner observes the state $X_{t, h}$, picks an action $A_{t,h}$ and 
observes the reward $r_{t,h}(X_{t, h}, A_{t,h})$. Then, unless $h=H$, the learner moves to the 
next state $X_{t,h+1}$, which is generated from the distribution $P(\cdot|X_{t, h}, A_{t,h})$.
At the end of step $H$, the episode terminates and a new one begins. The aim of the learner is to select its actions so 
that the cumulative sum of rewards is as large as possible.

Our algorithm and analysis will make use of the concept of (stationary stochastic) \emph{policies} 
$\pi: \ssp \to \Delta_{\A}$. A policy $\pi$ prescribes a behaviour rule to the learner 
by assigning probability $\pi(a|x)$ to taking action $a$ at state $x$. Let $\tau^{\pi} = ((X_1,A_1),(X_2,A_2),\dots,(X_H, A_H))$ be a trajectory generated by following the policy $\pi$ through the MDP. Then, for any $x_h\in \ssp_h, a_h \in \actions$ we define the occupancy measure $\mu_h^{\pi}(x,a) = \PPpi{(x,a) \in \tau^{\pi}}$.
We will refer to the collection of these distributions across all layers $h$ as the \emph{occupancy measure} 
induced by $\pi$ and denote it as $\mu^{\pi} = \pa{\mu^\pi_1,\mu^\pi_2,\dots,\mu^\pi_H}$. 
 We will denote the set of all valid occupancy measures by 
$\mathcal{U}$ and note that this is a convex set, such that for every element $\mu \in \mathcal{U}$  the following set of linear constraints is satisfied:
\begin{align}\label{eq:primal_constraints}
&\sum_{a\in \actions} \mu_{h+1}(x,a) = \sum_{x', a' \in \ssp_{h} \times \actions} P(x|x',a')\mu_h(x',a'), \quad \forall 
x\in\ssp_{h+1}, h \in [H-1],
\end{align}
as well as $\sum_a \mu_1(x_1,a) = 1$.
From every valid occupancy measure $\mu$, a stationary stochastic policy $\pi = {\pi_1, \dots, \pi_{H-1}}$ can be 
derived as $\pi_{\mu,h}(a|x) = \mu_h(x,a)/\sum_{a'}\mu_h(x,a')$.
For each $h$, introducing the linear operators $E$ and $P$ through their action on a set state-action distribution $u_h$ as $(E\transpose 
u_h)(x) = \sum_{a\in \actions} u_h(x,a)$ and $(P_h\transpose u_h)(x) = \sum_{x',a' \in \ssp_h, \actions} P(x|x',a') 
u_h(x',a')$, the constraints can be simply 
written as $E\transpose \mu_{h+1} = P\transpose_h \mu_h$ for each $h$. We will use the inner product notation for the sum over the set of states and actions: $\siprod{\mu_h}{r_h} = \sum_{(x,a)\in (\ssp_h\times\A)} \mu_{h}(x,a)r_{t,h}(x,a)$.
Using this notation, we formulate our objective as selecting a sequence of policies $\pi_t$ for each episode $t$ in a 
way that it minimizes the \emph{total expected regret} defined as
\[
 \regret_T = \sup_{\pi^*} \sum_{t=1}^T \sum_{h=1}^H \pa{\EEs{r_{t,h}(X_h^*,A_h^*)}{\pi^*} - 
\EEs{r_t(X_{t,h},A_{t,h})}{\pi_t}} = \sup_{\mu^* \in \U} \sum_{t=1}^T \sum_{h=1}^H \iprod{\mu_h^* - 
\mu_{h}^{\pi_t}}{r_{t,h}},
\]
where the notations $\EEs{\cdot}{\pi^*}$ and $\EEs{\cdot}{\pi_t}$ emphasize that the state-action trajectories are 
generated by following policies $\pi^*$ and $\pi_t$, respectively. As the above expression suggests, we can reformulate 
our online learning problem as an instance of online linear optimization where in each episode $t$, the learner selects 
an occupancy measure $\mu_t\in\mathcal{U}$ (with $\mu_t = \mu^{\pi_t}$) and gains reward $\sum_{h=1}^H 
\siprod{\mu_{t,h}}{r_{t,h}}$.  Intuitively, the regret measures the gap between the total reward gained by the learner and 
that of the best stationary policy fixed in hindsight, with full knowledge of the sequence of rewards chosen by the 
adversary. This performance measure is standard in the related literature on online 
learning in MDPs, see, for example \citet{NGS10, NIPS20134974, pmlr-v22-neu12, pmlrv97rosenberg19a,2019arXiv191205830C}.

In this paper, we focus on MDPs with potentially enormous state spaces, which makes it difficult to design 
computationally tractable algorithms with nontrivial guarantees, unless we make some assumptions.
We particularly focus on the classic technique of relying on \emph{linear function approximation} and assuming
that the reward functions occurring during the learning process can be written as a linear function of a 
low-dimensional feature map. We specify the form of function approximation and the conditions our analysis requires as 
follows:
\begin{assumption}[Linear MDP with adversarial rewards]\label{ass:linmdp_adverse}
	There exists a feature map $\varphi: \ssp \times \actions \to \real^d$  and a collection of $d$ signed measures $m 
= (m_{1},\dots,m_{d})$ on $\ssp$, such that for any $(x, a) \in \ssp \times \actions$ the transition function can be 
written as 
\[
P(\cdot|x,a) = \iprod{m(\cdot)}{\varphi(x,a)}.
\]
Furthermore, the reward function chosen by the adversary in each episode $t$ can be written as
\[
r_{t,h}(x, a) = \iprod{\theta_{t,h}}{\varphi(x,a)}
\]
for some $\theta_{t,h} \in \real^d$. 
We assume that the features and the parameter vectors satisfy $\norm{\varphi(x,a)} \le \sigma$ and that the first coordinate $\varphi_1(x,a) 
= 1$ for all $(x,a) \in \ssp\times\actions$. Also we assume that $\norm{\theta_{t,h}} \le R$. 
\end{assumption}
Online learning under this assumption, but with a fixed reward function, has received substantial attention in the recent literature, particularly since 
the work of \citet{2019arXiv190705388J} who popularized the term ``Linear MDP'' to refer to this 
class of MDPs. 
This has quickly become a common assumption for studying reinforcement learning algorithms 
(\citet{ 2019arXiv191205830C, 2019arXiv190705388J, NPB20,  agarwal2020flambe}).   This is also a
special case of \emph{factored linear models} (\citet{factlin,PS16}).

Linear MDPs come with several attractive properties that allow efficient optimization and learning. In this work, we 
will exploit the useful property shown by \citet{NPB20} and \citet{qreps} that all occupancy measures in a 
linear MDP can be seen to satisfy a relaxed version of the constraints in Equation~\eqref{eq:primal_constraints}. 
Specifically, for all $h$, defining the feature matrix $\Phi_h\in\real^{(\ssp_h\times\A)\times d}$ with its action on the distribution 
$u$ as $\Phi_h\transpose u = \sum_{x,a\in \ssp_h, \actions} u_h(x,a) \varphi(x,a)$, we define $\mathcal{U}_\Phi$ as the 
set of 
state-action distributions $(\mu,u)=\pa{(\mu_1,\dots,\mu_H),(u_1,\dots,u_H)}$ satisfying the following 
constraints:
\begin{equation}\label{eq:primal_relaxed}
 E\transpose u_{h+1} = P\transpose_h \mu_h \quad (\forall h), \qquad \Phi_h\transpose u_h = \Phi_h\transpose \mu_h \quad 
(\forall h), \qquad E\transpose u_1 = 1.
\end{equation}

 It is easy to see that for all feasible $(\mu,u)$ pairs, $u$ satisfies the original 
constraints~\eqref{eq:primal_constraints} if the MDP satisfies Assumption~\ref{ass:linmdp_adverse}: since the 
transition operator can be written as $P_h = \Phi_h M_h$ for some matrix $M_h$. In this case, we clearly have 
\begin{equation}\label{eupu}
	E\transpose u_{h+1} = P_h\transpose \mu_h = M_h\transpose \Phi_h\transpose \mu_h = M_h\transpose \Phi_h\transpose u_h = 
	P\transpose_h u_h,
\end{equation}
showing that any feasible $u$ is indeed a valid occupancy measure. Furthermore, due to linearity of the rewards in 
$\Phi$, we also have $\iprod{u_h}{r_{t,h}} = \iprod{\mu_h}{r_{t,h}}$ for all feasible $(\mu,u)\in\U_\Phi$.
While the number of variables and constraints in Equation~\eqref{eq:primal_relaxed} is still very large, 
it has been recently shown that approximate linear optimization over this set can be performed tractably 
\citep{NPB20,qreps}. Our own algorithm design described in the next section will heavily build on these recent results.

%% file: results.tex
\section{Algorithm and main results}\label{results}
This section presents our main contributions: a new efficient algorithm for the setting described above, along with its 
performance guarantees. Our algorithm design is based on a reduction to online linear optimization, exploiting the 
structural results established in the previous section. In particular, we will heavily rely on the algorithmic ideas 
established by \citet{qreps}, who proposed an efficient reduction of approximate linear optimization over the 
high-dimensional set $\mathcal{U}_\Phi$ to a low-dimensional convex optimization problem. Another key component of our 
algorithm is an efficient estimator of the reward vectors $\theta_{t,h}$ based on the work of \citet{2020NO}. For 
reasons that we will clarify in Section~\ref{sec:analysis}, accommodating these reward estimators into the framework of 
\citet{qreps} is not straightforward and necessitates some subtle changes. 

\subsection{The policy update rule}
Our algorithm is an instantiation of the well-known ``Follow the Regularized 
Leader'' (FTRL) template commonly used in the design of modern online learning methods (see, e.g., \citealp{Ora19}). We 
will make the following design choices:
\begin{itemize}
 \item The decision variables will be the vector $(\mu,u)\in\real^{2(\ssp\times\A)}$, with the feasible set $\U_\Phi^2$ 
defined through the constraints
\begin{equation}\label{eq:primal_relaxed_2}
 E\transpose u_h = P\transpose_h \mu_h \quad (\forall h), \qquad  \Phi_h \transpose \text{diag}(u_h) \Phi_h= 	 \Phi_h 
\transpose \text{diag}(\mu_h) \Phi_h \quad (\forall h).
\end{equation}
These latter constraints ensure that the feature covariance matrices under $u$ and $\mu$ will be identical, which is 
necessary for technical reasons that will be clarified in Section~\ref{sec:analysis}. 
Notice that, due to our assumption that $\varphi_1(x,a) = 1$, we have $\U_\Phi^2 \subseteq \U_{\Phi}$, so all feasible 
$u$'s continue to be feasible for the original constraints \eqref{eq:primal_constraints}.
 \item The regularization function will be chosen as $\frac 1\eta D(\mu\|\mu_0) + \frac 1\alpha \DC(u\|\mu_0)$ for some 
positive regularization parameters $\eta$ and $\alpha$, where $\mu_0$ is the occupancy measure induced by the uniform 
$\pi_0$ with $\pi_0(a|x) = \frac 1K$ for all $x,a$, and $D$ and $\DC$ are 
the marginal and conditional relative entropy functions respectively defined as $D(\mu\|\mu_0) = \sum_{h=1}^H 
D(\mu_h\|\mu_{0,h})$ and $\DC(\mu\|\mu_0) = \sum_{h=1}^H \DC(\mu_h\|\mu_{0,h})$ with
\begin{align*}
 D(\mu_h\|\mu_{0,h}) &= \sum_{(x,a)\in (\ssp_h\times\A)} \mu_h(x,a) \log \frac{\mu_h(x,a)}{\mu_{0,h}(x,a)}, 
\quad\mbox{and}
 \\
 \DC(\mu_h\|\mu_{0,h}) &= \sum_{(x,a)\in (\ssp_h\times\A)} \mu_h(x,a) \log \frac{\pi_{\mu,h}(a|x)}{\pi_{0,h}(a|x)}.
\end{align*}
\end{itemize}
With these choices, the updates of our algorithm in each episode will be given by
\begin{align}\label{opt}
\pa{\mu_t, u_t} &= \arg\max_{\pa{\mu, u} \in \mathcal{U}_\Phi^2} \bigg\{ \sum_{s=1}^{t-1} \sum_{h=1}^{H-1} \iprod{\mu_h}{ \widehat{r}_{s,h}  
}  - \frac{1}{\eta}D(\mu\| \mu_0 ) - \frac{1}{\alpha} \DC(u\| \mu_0) \bigg\}
\end{align}
where $\wh{r}_{t,h}\in\real^{\X\times\A}$ is an estimator of the reward function $r_{t,h}$ that will be defined shortly.

As written above, it is far from obvious if these updates can be calculated efficiently. The following result shows 
that, despite the apparent intractability of the maximization problem, it is possible to reduce the above problem into 
a $d^2$-dimensional unconstrained convex optimization problem:
\begin{proposition}\label{opt_solution}
	Define for each  $h \in[H-1]$, a matrix  $Z_h\in \real^{d\times d}$ and let  matrix $Z\in \real^{d\times d(H-1)}$ be defined as $Z = \pa{Z_1, \dots, Z_{H-1}}$. We will write $h(x) = h$, if $x \in \ssp_h$. Define the Q-function taking values $Q_Z(x,a) =  \varphi(x,a)\transpose Z_{h(x)} \varphi(x,a)$ and define the value function
	\begin{equation*}
		V_Z(x) = \frac{1}{\alpha}\log\pa{\sum_{a\in A(x)} \pi_0(a|x) e^{\alpha  Q_Z(x,a)}}
	\end{equation*}
	For any $h\in[H-1]$ and for any $x\in \mathcal{X}_h$, $a \in A(x)$, denote $P_{x,a}V_{Z} = \sum_{x' \in \mathcal{X}_{h(x)+1} }P(x'|x,a) V_{Z}(x')$ and  $\Delta_{t, Z}(x,a) =  \sum_{s=1}^{t-1}  \widehat{r}_{s,h(x) }(x,a) + P_{x,a}V_{Z} -  Q_Z(x,a) $. Then, the optimal solution of the optimization problem 
	(\ref{opt}) is given as 
	\begin{align*}
		\hpi_{t,h}(a|x)  &= \pi_0(a|x) e^{\alpha\pa{Q_{Z^*_{t}}(x,a) - V_{Z^*_{t}}(x)}},
		\\
		\hmu_{t,h}(x,a) &\propto \mu_0(x,a) e^{\eta\Delta_{t, Z^*_{t}}(x,a)},
	\end{align*}
	where $Z^*_t = (Z^*_{t,1}, \dots, Z^*_{t,H-1})$  is the minimizer of the convex function
	\begin{align}\label{potencial}
		\mathcal{G}_t( Z)= \frac{1}{\eta} \sum_{h=1}^{H-1} \log\pa{\sum_{x\in \mathcal{X}_h,a \in A(x)} \mu_0(x,a) e^{\eta \Delta_{t,Z}(x,a)}} +  V_{Z}(x_1).
	\end{align}
\end{proposition}
A particular merit of this result is that it gives an explicit formula for the policy $\pi_t$ that induces the 
optimal occupancy measure $u_t$, and that $\pi_t(a|x)$ can be evaluated straightforwardly as a function of the features 
$\varphi(x,a)$ and the parameters $Z_t^*$. The proof of the result is based on Lagrangian duality, and mainly 
follows the proof of Proposition~1 in \citet{qreps}, with some subtle differences due to the episodic setting 
we consider and the appearance of the constraints  $\Phi_h \transpose \text{diag}(u_h) \Phi_h=\Phi_h \transpose 
\text{diag}(\mu_h) \Phi_h$. The proof is presented in Appendix~\ref{app:opt_solution}. 

The proposition above inspires a very straightforward implementation that is presented as Algorithm~\ref{alg:linexp3}. 
Due to the direct relation with the algorithm of \citet{qreps}, we refer to this method as \OQREPS, where \QREPS stands 
for ``Relative Entropy Policy Search with Q-functions''. \OQREPS adapts the general idea of \QREPS to the online 
setting in a similar way as the \OREPS algorithm of \citet{NIPS20134974} adapted the Relative Entropy Policy Search 
method of \citet{PMA10} to regret minimization in tabular MDPs with adversarial rewards. While \OREPS would in 
principle be still applicable to the large-scale setting we study in this paper and would plausibly achieve similar 
regret guarantees, its implementation would be nearly impossible due to the lack of the structural properties enjoyed 
by \OQREPS, as established in Proposition~\ref{opt_solution}.

\begin{algorithm}[h]
	\caption{\ftrlrl}
	\label{alg:linexp3}
	\textbf{Parameters:} $\eta, \alpha>0$, exploration parameter $\gamma \in (0,1)$,\\
	\textbf{Initialization:} Set $\htheta_{1,h} = 0$ for all $h$, compute $Z_1$.
	\\
	\textbf{For} $t = 1, \dots, T$, \textbf{repeat:}
	\begin{itemize}
		\item Draw $Y_t \sim \mbox{Ber}(\gamma)$,
	\item \textbf{For} $h = 1, \dots, H$, \textbf{do:}
	\begin{itemize}
		\item Observe $X_{t,h}$ and, for all $a \in \actions(X_{t,h})$, set 
		\[ 
		\pi_{t,h}(a|X_{t,h}) = \pi_{0,h}(a|X_{t,h}) e^{\alpha\pa{Q_{Z_t}(X_{t,h},a)  - V_{Z_t}(X_{t,h}) }},
		\]
		\item if $Y=0$, draw $A_{t,h}\sim \pi_{t,h}(\cdot|X_{t,h})$, otherwise draw $A_{t,h}\sim 
\pi_{0,h}(\cdot|X_{t,h})$,
		\item observe the reward $\loss_{t,h}(X_{t,h},A_{t,h})$.
	\end{itemize}
	\item Compute $\htheta_{t,1}, \dots, \htheta_{t,H-1}$,  $Z_{t+1}$.
	\end{itemize}
\end{algorithm}

\subsection{The reward estimator}
We now turn to describing the reward estimators $\wh{r}_{t,h}$, which will require several further definitions.
Specifically, a concept of key importance will be the following \emph{feature covariance matrix}:
\[
\Sigma_{t,h} = \EEs{\varphi(X_{t,h}, A_{t,h} )\varphi(X_{t,h}, A_{t,h})\transpose}{\pi_t}.
\]
Making sure that $\Sigma_{t,h}$ is invertible, we can define the estimator
\begin{equation}\label{eq:lse}
	 \ttheta_{t,h} = \Sigma_{t,h}^{-1} \varphi(X_{t,h}, A_{t,h})r_{t,h}(X_{t,h}, A_{t,h}).
\end{equation}
This estimate shares many similarities with the estimates that are broadly used in the literature on adversarial linear 
bandits \citep{MB04,AK04,NIPS2007_3371}. 
It is easy to see that $ \ttheta_{t,h}$ is an unbiased estimate of $ \theta_{t,h}$:
 \begin{align*}
 	\EEt{ \ttheta_{t,h}} &= \EEt{ \Sigma_{t,h}^{-1} \varphi(X_{t,h}, A_{t,h}) \varphi(X_{t,h}, , A_{t,h})\transpose \theta_{t,h}
 	   } = \Sigma_{t,h}^{-1}  \Sigma_{t,h} \theta_{t,h}   = \theta_{t,h}.
 \end{align*}
Unfortunately, exact computation of $\Sigma_{t,h}$ is intractable. To address this issue, we propose a method to 
directly estimate the inverse of the covariance matrix $\Sigma_{t,h}$ by adapting the Matrix Geometric Resampling method 
of \citet{2020NO} (which itself is originally inspired by the Geometric Resampling method of 
\citealp{NeuBartok13,NB16}).
Our adaptation has two parameters $\beta >0$ and $M \in \intg_+$, and generates an estimate of the inverse covariance matrix through the following procedure\footnote{The version we present here is a na\"ive implementation, optimized for readability. We present a more practical variant in Appendix~\ref{appendixD}}:
\vspace{0.15cm}

\makebox[\textwidth][c]{
	\fbox{
		\begin{minipage}[l]{0.8\textwidth}
		\raggedright
			\textbf{Matrix Geometric Resampling}
			\vspace{.1cm}
			\hrule
			\vspace{.1cm}
			\textbf{Input:} simulator of $P$, policy $\wt \pi_t = (\wt \pi_{t,1}, \dots, \wt \pi_{t, H-1})$.\\
			\textbf{For $i = 1, \dots, M$, repeat}:
			\begin{enumerate}
				\vspace{-1mm}
				\item Simulate a trajectory $\tau(i) = \{ (X_1(i), A_1(i) ), \dots,  (X_{H-1}(i), A_{H-1}(i) ) \}$, following 
the policy $\wt \pi_t$ in $P$,
				\vspace{-1mm}
				\item \textbf{For $h = 1, \dots, H-1$, repeat}:\\
				Compute
				\begin{enumerate}
					\item  $B_{i, h} =  \varphi(X_h(i), A_h(i))\varphi(X_h(i), A_h(i))\transpose $,
					\item  $C_{i,h} = \prod_{j=1}^i  (I - \beta B_{j, h}) $.
				\end{enumerate}
			\end{enumerate}
			\vspace{-1mm}
			\textbf{Return} $\widehat{\Sigma}^{+}_{t,h} = \beta I+ \beta \sum_{i=1}^M C_{i,h}$  for all $h \in [H-1]$.
		\end{minipage}
	}
}

\vspace{0.1cm}
Based on the above procedure, we finally define our estimator as
 \[
 \htheta_{t,h} =  \widehat{\Sigma}^{+}_{t,h} \varphi(X_{t,h}, A_{t,h})r_{t,h}(X_{t,h}, A_{t,h}).
 \] 
The idea of the estimate is based on the truncation of the Neumann-series expansion of the matrix $\Sigma_{t,h}^{-1}$ at the $M$th order term.  Then, for large enough $M$, the matrix $\Sigma^+_{t,h}$ is a good estimator of the inverse covariance matrix, which will be 
 quantified formally in the analysis. For more intuition on the estimate, see section 3.2. in  \citet{2020NO}. With a 
careful implementation explained in Appendix~\ref{appendixD}, $\htheta_{t, h}$ can be computed in $O(MHKd)$ time, using  
$M$ calls to the simulator.

\subsection{The regret bound}\label{sec:regret_bound}
We are now ready to state our main result: a bound on the expected regret of \OQREPS. During the analysis, we will 
suppose that all the optimization problems solved by the algorithm are solved up to an additive error of 
$\varepsilon\ge 0$. Furthermore, we will 
denote the covariance matrix generated by the uniform policy at layer $h$ as $\Sigma_{0,h}$, and make the 
following assumption:
\begin{assumption}\label{ass:uniform}
 The eigenvalues of $\Sigma_{0,h}$ for all $h$ are lower bounded by $\lambdamin > 0$. 
\end{assumption}
Our main result is the following guarantee regarding the performance of \ftrlrl:
\begin{theorem}\label{th_regret}
Suppose that the  MDP satisfies Assumptions~\ref{ass:linmdp_adverse} and \ref{ass:uniform} and $\lambda_{\min} > 0$. 
Furthermore, suppose that, for all $t$, $Z_t$ satisfies $\GG_t(Z_t) \le \min_Z \GG_t(Z) + \varepsilon$ for some 
$\varepsilon\ge 0$.
Then, for $\gamma\in (0,1)$, 
$M \ge 0$, any positive $\eta \le \frac{2}{(M+2) H}$ and any positive $\beta \le \frac{1}{2\sigma^2}$, the expected 
regret of \ftrlrl over $T$ episodes satisfies
\begin{align*}
		\regret_T \le& 2T \sigma R H\cdot \exp\pa{-\gamma\beta\lambdamin M} + \gamma  H T 
    + \eta HdT\frac{ 4}{3} +  \frac{1}{\eta}D(\mu^*\|\mu_{0}) + \frac{1}{\alpha}
\DC(u^*\|\mu_{0})
\\
&+ \sqrt{\alpha \varepsilon }(M+2) H T.
\end{align*}
Furthermore, letting $\beta = \frac{1}{2\sigma^2}$, $M = \left\lceil\frac{ 2 \sigma^2 \log (T H
	\sigma R)}{\gamma  \lambdamin}\right\rceil$, $\eta = \frac{1}{\sqrt{TdH }}$, $\alpha = \frac{1}{\sqrt{TdH }}$ and 
$\gamma = \frac{1}{\sqrt{TH}}$ and supposing that $T$ is large enough so 
	that the above constraints on $M, \gamma, \eta$ and $\beta$ are satisfied, we also have 
\begin{align*}
\regret_T \le&   \sqrt{dHT} \pa{2 + D(\mu^*\|\mu_0) + \DC(u^*\|\mu_0)} + \sqrt{HT} + \sqrt{ \varepsilon 
}T^{5/4}(Hd)^{1/4} + 2.
%
\end{align*}
\end{theorem}
Thus, when all optimization problems are solved up to precision $\varepsilon = T^{-3/2}$, the regret of \OQREPS is 
guaranteed to be of $\OO\bpa{\sqrt{dHT D(\mu^*\|\mu_0)}}$.

\input{implementation}

%% file: implementation.tex
\subsection{Implementation}\label{sec:implementation}
While Proposition~\ref{opt_solution} establishes the form of the ideal policy updates $\pi_t$ through the solution of 
an unconstrained convex optimization problem, it is not obvious that this optimization problem can be 
solved efficiently. Indeed, one immediate challenge in optimizing $\GG_t$ is that its gradient takes the form
\[
 \nabla \GG_t(Z) = \sum_{x,a} \tmu_Z(x,a) \pa{\varphi(x,a)\varphi(x,a)\transpose - \sum_{x',a'} P(x'|x,a) 
\pi_Z(a'|x') \varphi(x',a')\varphi(x',a')\transpose},
\]
where $\tmu_Z(x,a) = \frac{\mu_0(x,a) \exp(\eta \Delta_Z(x,a))}{\sum_{x',a'}\mu_0(x',a')  \exp(\eta \Delta_Z(x',a'))}$.
Sampling from this latter distribution (and thus obtaining unbiased estimators of $\nabla \GG_t(Z)$) is 
problematic due to the intractable normalization constant.

This challenge can be addressed in a variety of ways. First, one can estimate the gradients via weighted importance 
sampling from the distribution $\tmu_Z$ and using these in a stochastic optimization procedure. This approach has been 
recently proposed and analyzed for an approximate implementation of REPS by \citet{PLBN21}, who showed that it results 
in $\varepsilon$-optimal policy updates given polynomially many samples in $1/\varepsilon$.  Alternatively, one can 
consider an empirical counterpart of the loss function replacing the expectation with respect to $\mu_0$ with an 
empirical average over a number of i.i.d.~samples drawn from the same distribution. The resulting loss function can then 
be optimized via standard stochastic optimization methods. This approach has been proposed and analyzed by 
\citet{qreps}. We describe the specifics of this latter approach in Appendix~\ref{app:implementation}.

%% file: analysis.tex
\section{Analysis}\label{sec:analysis}
This section gives the proof of Theorem~\ref{th_regret} by stating the main technical results as lemmas and putting 
them together to obtain the final bound. 
In the first part of the proof, we show the upper bound on the auxiliary regret minimization game with general reward 
inputs and ideal updates. Then, we relate this quantity to the true expected regret by taking into account the 
properties of our reward estimates and the optimization errors incurred when calculating the updates. The proofs of all 
the lemmas are deferred to Appendix~\ref{app:proofs1}.

We start by defining the idealized updates $(\hmu_t,\hu_t)$ obtained by solving the update steps in 
Equation~\eqref{opt} exactly, and we let $u_t$ be the occupancy measure induced by policy $\pi_t$ that is based on the 
near-optimal parameters $Z_t$ satisfying $\GG_t(Z_t) \le \min_Z \GG_t(Z) + \varepsilon$. We will also let $\mu_t$ 
be the occupancy measure resulting from mixing $u_t$ with the exploratory distribution $\mu_0$ and note that $\mu_{t,h} = 
(1-\gamma) u_{t,h} + \gamma \mu_{t,h}$.
Using this notation, we will 
consider an auxiliary online learning problem with the sequence of reward functions given as $\hr_{t,h}(x,a) = 
\siprod{\varphi(x,a)}{\htheta_{t,h}}$, and study the performance of the idealized sequence $(\hmu_t,\hu_t)$ therein:
\[
\hregret_{T} = \sum_{t=1}^T\sum_{h=1}^{H-1}  \siprod{\mu^*_h - \hu_{t,h}}{\wh r_{t,h}}.
\]
Our first lemma bounds the above quantity:
\begin{lemma}\label{exp2base}
Suppose that $\htheta_{t,h}$ is such that $\big|\eta \cdot \siprod{\varphi(x,a)}{\htheta_{t,h}} \big| < 1$ holds for all 
$x,a$. 
Then, the auxiliary regret satisfies
	\begin{align*}
	\hregret_{T}   &\le   \eta \sum_{t=1}^T \sum_{h=1}^{H-1}\siprod{\hmu_{t,h}}{\wh r_{t,h}^2}+ 
\frac{1}{\eta}D(\mu^*\|\mu_{0}) + \frac{1}{\alpha}\DC(u^*\|\mu_{0}).
\end{align*}
\end{lemma}
While the proof makes use of a general potential-based argument commonly used for analyzing FTRL-style algorithms, it 
involves several nontrivial elements exploiting the structural results concerning \OQREPS proved in 
Proposition~\ref{opt_solution}. In particular, these properties enable us to upper bound the potential differences in a 
particularly simple way. 
The main term on contributing to the regret $\hregret_T$ can be bounded as follows:
\begin{lemma}\label{quadratic} 
Suppose that $\varphi(X_{t,h},a)$ is satisfying $\twonorm{\varphi(X_{t,h}, a)} \le \sigma $ for any $a$, $0 < \beta 
\le \frac{1}{2\sigma^2}$ and $M > 0$. Then for each $t$ and $h$,
	\begin{align*}
\EEt{\siprod{\hmu_{t,h}}{\wh r_{t,h}^2} } \le \frac{4 d}{3(1-\gamma)} + (M+1)^2\onenorm{\hu_{t,h} - u_{t,h}}.
	\end{align*}
\end{lemma}
The proof of this claim makes heavy use of the fact that $\siprod{\hmu_{t,h}}{\wh r_{t,h}^2} = \siprod{\hu_{t,h}}{\wh 
r_{t,h}^2}$, which is ensured by the construction of the reward estimator $\hr_{t,h}$ and the constraints on the 
feature covariance matrices in Equation~\eqref{eq:primal_relaxed_2}. This property is not guaranteed to hold under the 
first-order constraints~\eqref{eq:primal_relaxed} used in the previous works of \citet{NPB20} and \citet{qreps}, which 
eventually justifies the higher complexity of our algorithm. 

It remains to relate the auxiliary regret to the actual regret. The main challenge is accounting for the 
mismatch between $\mu_t$ and $u_t$, and the bias of $\hr_t$, denoted as $b_{t,h}(x,a) = \EEt{\hr_{t,h}(x,a)} - 
r_{t,h}(x,a)$. 
To address these issues, we observe that for any $t,h$, we have
\begin{align*}
 \iprod{\mu_{t,h}}{r_{t,h}} &= \iprod{(1-\gamma) u_{t,h} + \gamma \mu_{0,h}}{r_{t,h}}  =
 \iprod{(1-\gamma) \hu_{t,h} + \gamma \mu_{0,h}}{r_{t,h}} + (1-\gamma)\iprod{u_{t,h}-\hu_{t,h}}{r_{t,h}}
 \\
 &\ge 
 \EEt{\iprod{(1-\gamma) \hu_{t,h} + \gamma \mu_{0,h}}{\hr_{t,h}}} + \infnorm{b_{t,h}} + 
(1-\gamma)\onenorm{u_{t,h}-\hu_{t,h}},
\end{align*}
where in the last step we used the fact that $\infnorm{r_{t,h}} \le 1$.
After straightforward algebraic manipulations, this implies that the regret can be bounded as
\begin{align}\label{eq:almostbound}
\regret_T &\le (1-\gamma)\EE{\hregret_T} + \sum_{t=1}^T \sum_{h=1}^H \EE{\gamma \iprod{\mu_{0,h} - \mu^*_h}{r_{t,h}}
+ \onenorm{\hu_{t,h} - u_{t,h}} + \infnorm{b_{t,h}}}.
\end{align}

In order to proceed, we need to verify the condition $\big|\eta \cdot \siprod{\varphi(x,a)}{\htheta_{t,h}} \big| < 1$ 
so that we can apply Lemma~\ref{exp2base} to bound $\hregret_T$. This is done in the following lemma:
\begin{lemma}\label{lem:bounded}
Suppose that $\eta \le \frac{2}{(M+2)}$. Then, for all, $t,h$, the reward estimates satisfy 
$\eta \infnorm{\hr_{t,h}} < 1$.
\end{lemma}
Proceeding under the condition $\eta (M+1)$, we can apply Lemma~\ref{exp2base} to bound the first term on the 
right-hand side of Equation~\eqref{eq:almostbound}, giving
\[
 \regret_T \le \frac{D(\mu^*\|\mu_{0})}{\eta} + \frac{\DC(u^*\|\mu_{0}) }{\alpha}
 + \frac{4\eta d H T}{3} + 
\gamma HT + 
\sum_{t,h}\EE{(M+2)\onenorm{\hu_{t,h} - u_{t,h}} + \infnorm{b_{t,h}}}.
\]

It remains to bound the bias of the reward estimators and the effect of the optimization errors that result in the 
mismatch between $u_t$ and $\hu_t$. The following lemma shows that this mismatch can be directly controlled as a 
function of the optimization error: 
\begin{lemma}\label{lem:propagation}
 The following bound is satisfied for all $t$ and $h$: $\onenorm{\hu_{t,h} - u_{t,h}} \le \sqrt{2 \alpha \varepsilon}$.
\end{lemma}
The final element in the proof is the following lemma that bounds the bias of the estimator:
\begin{lemma}\label{bias}
	For $M \ge 0$, $\beta = \frac{1}{2\sigma^2}$, we have $\infnorm{b_{t,h}} \le \sigma R \exp\pa{ -\gamma \beta 
\lambdamin M}$.

\end{lemma}
Putting these bounds together with the above derivations concludes 
the proof of Theorem~\ref{th_regret}.

%% file: discussion.tex
\section{Discussion}\label{sec:discussion}
This paper  studies the problem of online learning in MDPs, merging two important lines of work on this problem concerned with 
linear function approximation \citep{2019arXiv190705388J,2019arXiv191205830C} and bandit feedback with 
adversarial rewards \citep{NGS10,NGS13,NIPS20134974}. Our results are the first in this setting and not directly 
comparable with any previous work, although some favorable comparisons can be made with previous results in related 
settings. In the tabular setting where $d = |\ssp||\A|$, our bounds exactly recover the minimax optimal guarantees 
first achieved by the \OREPS algorithm of \citet{NIPS20134974}. For realizable linear function approximation, the 
work closest to ours is that of \citet{2019arXiv191205830C}, who prove bounds of order $\sqrt{d^2 H^3 T}$, which is 
worse by a factor of $\sqrt{d} H$ than our result. Their setting, however, is not exactly comparable to ours due to the 
different assumptions about the feedback about the rewards and the knowledge of the transition function.

One particular strength of our work is providing a complete analysis of the propagation of optimization errors incurred 
while performing the updates. This is indeed a unique contribution in the related literature, where the effect of such 
errors typically go unaddressed. Specifically, the algorithms of \citet{NIPS20134974}, 
\citet{pmlrv97rosenberg19a}, and \citet{JJLSY20} are all based on solving convex optimization problems similar to ours, 
the effect of optimization errors or potential methods for solving the optimization problems are not discussed at all. 
That said, we believe that the methods for calculating the updates discussed in Section~\ref{sec:implementation} are 
far from perfect, and more research will be necessary to find truly practical optimization methods to solve this 
problem.

The most important open question we leave behind concerns the requirement to have full prior knowledge of 
$P$. In the tabular case, this challenge has been successfully addressed in the adversarial MDP problem recently by 
\citet{JJLSY20}, whose technique is based on adjusting the constraints~\eqref{eq:primal_constraints} with a confidence 
set over the transition functions, to account for the uncertainty about the dynamics. We find it plausible that a 
similar extension of \OQREPS is possible by incorporating a confidence set for linear MDPs, as has been done in the case 
of i.i.d.~rewards by \citet{NPB20}. Nevertheless, the details of such an extension remain highly non-trivial, and we 
leave the challenge of working them out open for future work.

%% file: appendixA.tex
\appendix
\section{Omitted proofs}\label{app:proofs1}
\subsection{The proof of Proposition~\ref{opt_solution}}\label{app:opt_solution}
The proof is based on Lagrangian duality: for each $h \in[H-1]$, we introduce a set of multipliers $V_h\in 
\real^{|X_h|}$ and $Z_h\in \real^{d\times d}$ corresponding to the two sets of constraints connecting $\mu_{t,h}$ and 
$u_{t,h}$, and $\rho_{t,h}$ for the normalization constraint of $\mu_{t,h}$. 
 Then, we can write the Lagrangian of the constrained
optimization problem as
\begin{align*}
	\mathcal{L}(\mu, u; V, Z, \rho) =&  \sum_{h=1}^{H-1}  \sum_{s=1}^{t-1} \siprod{\mu_h}{ \widehat{r}_{s,h}} +  
\iprod{ Z_h}{ \Phi_h\transpose  (\text{diag}(u_h) - \text{diag}(\mu_h)) \Phi_h}  \\
	& + \sum_{h=1}^{H-1} \pa{\rho_h(1-\iprod{\mu_h}{\bone}) - \frac{1}{\eta}D(\mu_h\|\mu_{0,h}) - 
\frac{1}{\alpha}\DC(u_h\|\mu_{0,h}) }\\
	& + V_1(x_1)(1- E\transpose u_1) + \sum_{h=1}^{H-1}\siprod{V_{h+1}}{P\transpose\mu_h - E\transpose u_{h+1}}.
\end{align*}
For any $h\in[H-1]$, for  any $x\in \mathcal{X}_h, a \in A(x)$, denote $Q_{Z}(x,a) = \varphi(x,a)\transpose Z_{h(x)} \varphi(x,a) $,  $P_{x,a}V_{h+1} = \sum_{x' \in \mathcal{X}_{h+1} }P(x'|x,a) V_{h+1}(x')$ and  $\Delta_{t, Z}(x,a) =  \sum_{s=1}^{t-1}  \widehat{r}_{s,h(x) }(x,a) + P_{x,a}V_{h(x)+1} -  Q_{Z}(x,a) $. The above Lagrangian is strictly concave, so the maximum of $\mathcal{L}(\mu, d; V, Z, \rho)$ can be found by setting the derivatives with respect to its parameters to zero. This gives the following expressions for the choices of $\pi$ and $\mu$:
$$\pi_{t,h}^*(a|x) = \pi_{0,h}(a|x) e^{\alpha\pa{ Q_Z(x,a)  - V_h(x)}},$$ 
$$\mu^*_{t,h}(x,a) = \mu_0(x,a) e^{\eta(\Delta_{ t,Z}(x,a) - \rho_{t,h})},$$ 
From the constraint $\sum_{x\in \mathcal{X}_h,a\in A(x)} \mu^*_{t,h}(x,a) = 1$ for all $h$, we get that  $$\rho^*_{t,h} = \frac{1}{\eta} \log\pa{\sum_{x\in \mathcal{X}_h,a\in A(x)} \mu_0(x,a) e^{\eta \Delta_{t,Z}(x,a)}}$$ and from the constraint $\sum_a \pi_t^*(a|x)  =1$, we get 
\[ V_h^*(x) = \frac{1}{\alpha} \log\pa{ \sum_a \pi_0(a|x)e^{\alpha Q_{Z}(x,a)}  }. \] 
We will further use the notation $  V_{Z}(x) := V_h^*(x) $. Then, by plugging $\pi_{t,h}^*, \mu^*_{t,h}, V_{Z}(x)$ into the Lagrangian, we get
\begin{align*}
	\mathcal{G}_t(Z)=\mathcal{L}(\mu^*, u^*; V^*, Z, \rho^*)  =  \frac{1}{\eta} \sum_{h=1}^{H-1} \log\pa{\sum_{x\in \mathcal{X}_h,a \in A(x)} \mu_0(x,a) e^{\eta \Delta_{t,Z}(x,a)}} +  V_{Z}(x_1).
\end{align*}
Then, the solution of the optimization problem can be written as
\begin{align*}
	\max_{\mu,u \in U} \min_{V,Z, \rho} \mathcal{L}(\mu, u; V, Z, \rho)  = \min_{V,Z, \rho} \max_{\mu,u \in U}  
\mathcal{L}(\mu, u; V, Z, \rho) = \min_{Z} \mathcal{L}(\mu^*, u^*; V^*, Z, \rho^*) = \min_Z \GG_t(Z).
\end{align*}
This concludes the proof.
\jmlrQED

\subsection{The proof of Lemma~\ref{exp2base}}\label{app:exp2base}
The proof is based on a variation of the FTRL analysis that studies the evolution of the potential function $\Psi_t$ 
defined for each $t$ as
\[
 \Psi_t = \max_{(\mu, u)\in \mathcal{U}^2_{\Phi}} \biggl\{ \sum_{s=1}^{t-1}\sum_{h=1}^H \siprod{\mu_h}{\hr_{s,h}}  - 
\frac{1}{\eta}D(\mu\|\mu_0)-\frac{1}{\alpha}  \DC(u\|u_0)  \biggr\}.
\]
This definition immediately implies the following bound:
\begin{equation}\label{eq:potential_below}
 \Psi_{T+1} \ge \sum_{s=1}^T\sum_{h=1}^{H-1} \siprod{\mu^*_h}{ \hr_{s,h}}  - \frac{1}{\eta}D(\mu^*\|\mu_0) - 
\frac{1}{\alpha}\DC (u^*\|u_0).
\end{equation}

To proceed, we will heavily exploit the fact that, by Proposition~\ref{opt_solution}, the potential satisfies $\Psi_t = 
\min_Z \GG_t$. Introducing the notation $Z_t^* = \argmin_Z \GG_t(Z)$, we have
\begin{align*}
 &\Psi_{t+1} - \Psi_t = \mathcal G_{t+1}(Z^*_{t+1})- \mathcal G_{t}(Z^*_{t}) \le \mathcal G_{t+1}(Z^*_{t})- \mathcal 
G_{t}(Z^*_{t})  \\
	&\quad = \frac{1}{\eta} \sum_{h=1}^{H-1} \log \frac{\sum_{x\in \ssp_h, a \in \actions} \mu_{0,h}(x,a)\exp\pa{\eta 
\pa{ \sum_{s=1}^{t} \hr_{s,h}(x,a) +P_{x,a}V_{Z^*_{t}} -Q_{Z^*_{t}}(x,a) } }}{ \sum_{x'\in \ssp_h,a' \in \actions} 
\mu_{0,h}(x',a')\exp\pa{\eta\pa{\sum_{s=1}^{t-1} \hr_{s,h}(x,a) + P_{x',a'}V_{Z^*_{t}}-Q_{Z^*_{t}}(x',a')  }} }
\\
	&\quad = \frac{1}{\eta} \sum_{h=1}^{H-1} \log \pa{ \sum_{x\in \mathcal{X}_h,a \in \actions}   \mu_{t,h}(x,a)\exp\pa{ 
\eta \wh r_{t,h}(x,a)} }\\
&\qquad\qquad\mbox{(using the expression of $\mu_{t,h}(x,a)$ obtained in Proposition~\ref{opt_solution})}
\\
	&\quad\le \frac{1}{\eta} \sum_{h=1}^{H-1} \log \pa{ 1 + \sum_{x\in \mathcal{X}_h,a \in \actions}  
\mu_{t,h}(x,a) \eta \pa{\wh r_{t,h}(x,a) + \eta \wh r^2_{t,h}(x,a)}  }
\\
	&\quad \le  \sum_{h=1}^{H-1} \pa{\siprod{\mu_{t,h}}{\wh r_{t,h}}+ \eta\siprod{\mu_{t,h}}{\wh r_t^2}},
\end{align*}
where in the last two lines we have used the inequalities $e^z \le 1 + z + z^2$, which holds for $z \le 1$ and  
$\log(1+z)\le z$, which holds for all $z > -1$, which conditions are verified due to our constraint on $\eta$.

Summing up both sides for all $t$ and combining the result with the inequality~\eqref{eq:potential_below}, 
we obtain
\begin{align*}
\hregret_T = \sum_{s=1}^T\sum_{h=1}^{H-1} \siprod{\mu^*_h}{ \hr_{s,h}}  - \sum_{t=1}^T \siprod{\mu_t}{\wh r_t}   \le 
\eta \sum_{t=1}^T \sum_{h=1}^{H-1} \siprod{\mu_{t,h}}{\wh r^2_{t,h}} + \frac{1}{\eta}D(\mu^*||\mu_0) + 
\frac{1}{\alpha}\DC (u^*||u_0),
\end{align*}
concluding the proof.
\jmlrQED

%% file: appendixB.tex
\subsection{The proof of Lemma~\ref{quadratic}}
\label{appendixA}
We start by using the the definition of $\htheta_{t,h}$ to obtain
\begin{align*}
&\EEt{\sum_{x\in\mathcal{X}_h, a\in \A}   \hmu_{t,h}(x,a)  \biprod{\varphi(x, 
a)}{\htheta_{t,h}}^2} = \EEt{\sum_{x\in\mathcal{X}_h, a\in \A}   \hmu_{t,h}(x,a)  
\trace{\varphi(x,a)\varphi(x, a)\transpose \htheta_{t,h} \htheta_{t,h}\transpose}}
\\
&\qquad = \EEt{\sum_{x\in\mathcal{X}_h, a\in \A}   \hu_{t,h}(x,a)
\trace{\varphi(x,a)\varphi(x, a)\transpose \htheta_{t,h} \htheta_{t,h}\transpose}} 
\\
&\qquad\qquad\qquad\qquad\qquad\qquad\mbox{(by the constraint $\Phi\transpose_h \text{diag}(\hmu_t) \Phi_h = 
\Phi\transpose_h \text{diag}(\hu_t) \Phi_h$)} \nonumber
\\
&\qquad=
\EEt{\sum_{x\in\mathcal{X}_h, a\in \A}   u_{t,h}(x,a) \trace{\varphi(x,a)\varphi(x, a)\transpose \htheta_{t,h} 
\htheta_{t,h}\transpose}} 
\\
&\qquad\qquad \qquad\qquad + \sum_{x\in\mathcal{X}_h, a\in \A}   \pa{u_{t,h}(x,a) - \hu_{t,h}(x,a)}  
\EEt{\biprod{\varphi(x, 
a)}{\htheta_{t,h}}^2}
\\
&\qquad\le
\EEt{\sum_{x\in\mathcal{X}_h, a\in \A}   u_{t,h}(x,a) \trace{\varphi(x,a)\varphi(x, a)\transpose \htheta_{t,h} 
\htheta_{t,h}\transpose}} + \onenorm{u_{t,h}- \hu_{t,h}} \cdot \infnorm{\EEt{\hr_{t,h}^2}}
\end{align*}
The second term can be bounded straightforwardly by $\onenorm{u_{t,h}- \hu_{t,h}} (M+1)^2$, using
Lemma~\ref{lem:bounded} to bound $\infnorm{\hr_{t,h}} \le (M+1)$. 
As for the first term, we have
\begin{align*}
&(1-\gamma)\EEt{\sum_{x\in\mathcal{X}_h, a\in \A}   u_{t,h}(x,a) \trace{\varphi(x,a)\varphi(x, a)\transpose 
\htheta_{t,h}\htheta_{t,h}\transpose}} \\
&\quad \le (1-\gamma)\EEt{\sum_{x\in\mathcal{X}_h, a\in \A} \trace{  u_{t,h}(x,a) \varphi(x, a)\varphi(x, 
a)\transpose \hSp_{t,h}  \varphi(X_{t,h}, A_{t,h}) \varphi(X_{t,h}, A_{t,h})\transpose \hSp_{t,h} \nonumber
}},
\\
& \quad  \le  (1-\gamma)\EEt{\sum_{x\in\mathcal{X}_h, a\in \A} \trace{  u_{t,h}(x,a) \varphi(x, a)\varphi(x, 
a)\transpose \hSp_{t,h}  \varphi(X_{t,h}, A_{t,h}) \varphi(X_{t,h}, A_{t,h})\transpose \hSp_{t,h} \nonumber
}}\\
&\qquad\qquad + \gamma \EEt{\sum_{x\in\mathcal{X}_h, a\in \A} \trace{  u(x,a) \varphi(x, a)\varphi(x, a)\transpose 
\hSp_{t,h}  \varphi(X_{t,h}, A_{t,h}) \varphi(X_{t,h}, A_{t,h})\transpose \hSp_{t,h} \nonumber
}}\\
& \quad  =  \EEt{ \trace{  \Sigma_{t,h} \hSp_{t,h} \Sigma_{t,h}\hSp_{t,h} \nonumber
}},
\end{align*}
where we used $\left|r_{t,h}(X_{t,h},A_{t,h})\right| \le 1 $ in the first inequality. For ease of readability, we will 
omit the indices $h$ in the rest of the proof.
Using the definition of 
$\Sp_{t}$ and elementary manipulations, we get
\begin{align*}
&  \EEt{\trace{\Sigma_t \Sp_{t} \Sigma_{t} \Sp_{t} }}= \beta^2 \cdot \EEt{ \trace{\Sigma^*   \pa{\sum_{k=0}^M C_{k}} \Sigma_{t}
		\pa{\sum_{j=0}^M C_{j}}} }
\\
&\qquad= \beta^2 \EEt{ \sum_{k=0}^M \sum_{j=0}^M\trace{ \Sigma_t  C_{k} \Sigma_{t}  C_{j}} }
\\
&\qquad = 
\beta^2 \EEt{\sum_{k=0}^M \trace{\Sigma_t C_{k} \Sigma_{t} C_{k}}} + 2\beta^2 \EEt{ \sum_{k=0}^M \sum_{j= k+1}^M \trace{\Sigma_t 
		C_{k} \Sigma_{t} C_{j}}}.
\end{align*}
Let us first address the first term on the right hand side. To this end, consider any symmetric positive definite matrix $S$ that commutes with $\Sigma_{t}$ and observe that
\begin{align*}
&\EEt{(I- \beta B_{k} ) S (I - \beta B_{k} )} \\
&\qquad = \EE{(I - \beta \varphi(X(k), A(k))\varphi(X(k), A(k))\transpose ) S (I - 
	\beta \varphi(X(k), A(k))\varphi(X(k), A(k))\transpose )} \\
&\qquad = S - \beta\EE{\varphi(X(k), A(k))\varphi(X(k), A(K))\transpose   S} - \beta\EEt{S \varphi(X(k), A(k))\varphi(X(k), A(k))\transpose  } \\
&\qquad + \beta^2 \EEt{\varphi(X(k), A(k))\varphi(X(k), A(k))\transpose S \varphi(X(k), A(k))\varphi(X(k), A(k))\transpose }\\
&\qquad  \preccurlyeq S- 2\beta S\Sigma_{t}  + \beta^2 \sigma^2 S\Sigma_{t} = S \pa{I - \beta (2-\beta\sigma^2)\Sigma_{t}},
\end{align*}
where we used our assumption that $\norm{\varphi(X(k),  A(k))}\le \sigma$, which implies $\EEt{\twonorm{\varphi(X(k),  A(k))}^2 \varphi(X(k),  A(k))\varphi(X(k),  A(k))\transpose } \preccurlyeq 
\sigma^2 \Sigma_{t}$.
Now, recalling the definition $C_{k} = \prod_{j=1}^k (I - \beta B_{j}) $ and using the above relation repeatedly, we can obtain
\begin{equation}\label{eq:SASAbound}
\begin{split}
\trace{\EEt{\Sigma_t  C_{k} \Sigma_{t} C_{k}} } &= \trace{\EEt{\Sigma_t  C_{k-1} \EEt{(I - \beta  B_{k}) \Sigma_{t} (I- \beta B_{k})}C_{k-1}}} 
\\
&\le \trace{\EEt{\Sigma_t  C_{k-1} \Sigma_{t} \pa{I - \beta (2-\beta\sigma^2)\Sigma_{t}} C_{k-1}}} 
\\
&\le \ldots \le  \trace{\Sigma_t  \Sigma_{t} (I - \beta (2-\beta \sigma^2)\Sigma_{t} )^{k}}.
\end{split}
\end{equation}
Thus, we can see that
\begin{align*}
&\beta^2\sum_{k=0}^M \trace{\EEt{\Sigma_t   C_{k} \Sigma_{t} C_{k}}} = \beta^2\sum_{k=0}^M \trace{\Sigma_t \Sigma_{t} (I - 
	\beta (2-\beta \sigma^2)\Sigma_{t})^{k}} \\
&\qquad= \frac{\beta^2}{\beta(2-\beta\sigma^2)} \trace{\Sigma_t  \Sigma_{t} \Sigma_{t}^{-1} \pa{I - (I - 
		\beta(2-\beta \sigma^2)\Sigma_{t})^M}} \le \frac{\beta \trace{\Sigma_t }}{2-\beta\sigma^2} \le \frac{2\beta \trace{\Sigma_t }}{3},
\end{align*}
where we used the condition $\beta \le \frac{1}{2\sigma^2}$ and the fact that $(I - \beta(2-\beta \sigma^2)\Sigma_{t} )^M \succcurlyeq 0$ by the same condition.
We can finally observe that our assumption on the contexts implies $\trace{\Sigma_t }\le \trace{\sigma^2 I} = \sigma^2 d$, so again by our condition on $\beta$ we have $\beta \trace{\Sigma_t} \le \frac d2$, and the first term is bounded by $\frac d3$.

Moving on to the second term, we first note that for any $j>k$, the conditional expectation of $B_{j}$ given $B_{\le k} = (B_{1},B_{2},\dots B_{k})$ satisfies 
$\EEcc{C_{k}}{B_{\le k}} = C_{k} (I - \beta\Sigma)^{j-k}$ due to conditional independence of all $B_{j}$ given $B_{k}$, for $i>k$. We make use of this equality by writing
\allowdisplaybreaks
\begin{align*}
&\beta^2\sum_{k=0}^M \sum_{j= k+1}^M \EE{\trace{\Sigma_t  C_{k} \Sigma_{t} C_{j}}} 
= \beta^2\sum_{k=0}^M \EE{\EEcc{\sum_{j= k+1}^M \trace{\Sigma_t C_{k} \Sigma_{t} C_{j}}}{B_{\le k}}} \nonumber
\\
&\qquad\qquad= \beta^2\sum_{k=0}^M \EE{\EEcc{\sum_{j= k+1}^M \trace{\Sigma_t  C_{k} \Sigma_{t} C_{j} (I - 
			\beta\Sigma_{t})^{j-k}}}{B_{\le k}}} \nonumber
\\
&\qquad\qquad= \beta\sum_{k=0}^M \EE{\EEcc{\trace{\Sigma_t C_{k}\Sigma_{t} C_{k} \Sigma^{-1}_{t}\pa{I - (I - 
				\beta\Sigma_{t})^{M-k}}}}{B_{\le k}}} \nonumber
\\
&\qquad\qquad\le \beta \sum_{k=0}^M \EE{\EEcc{\trace{\Sigma_t  C_{k} \Sigma_{t}  C_{k} \Sigma^{-1}_{t} }}{B_{\le k}}} \nonumber
\\
&\qquad\qquad\qquad\qquad\qquad\qquad\mbox{(due to $(I - \beta\Sigma_{t} )^{M-k} \succcurlyeq 0$)} \nonumber
\\
&\qquad\qquad\le \beta \sum_{k=0}^M \trace{\Sigma_t  \Sigma_{t}  (I - \beta (2-\beta \sigma^2)\Sigma_{t})^{k}\Sigma_{t}^{-1}}  \nonumber
\\
&\qquad\qquad\qquad\qquad\qquad\qquad\mbox{(by the same argument as in Equation~\eqref{eq:SASAbound})} \nonumber
\\
&\qquad\qquad\le \frac{1}{(2-\beta \sigma^2)} \trace{\Sigma_t \Sigma_{t} \Sigma_{t}^{-1} \pa{I - (I - 
		\beta(2-\beta \sigma^2)\Sigma_{t})^M \Sigma_{t}^{-1}}} \nonumber
\\
&\qquad\qquad\le \trace{\Sigma_t  \Sigma_{t}^{-1}\Sigma_t  \Sigma_{t}^{-1}} = d.\nonumber
\end{align*}
The proof of the lemma is finished by putting everything together. 
\jmlrQED

\subsection{The proof of Lemma~\ref{bias}}\label{bias_lemma}

We first observe that the bias of $\htheta_{t,h}$ can be easily expressed as
\begin{align*}
\EEtb{ \htheta_{t,h} }
&= \EEt{\widehat{\Sigma}^{+}_{t,h} \varphi(X_{t,h}, A_{t,h})  \varphi(X_{t,h}, A_{t,h})\transpose \theta_{t,h} }  
\\
&= \EEt{\widehat{\Sigma}^{+}_{t,h} } \EEt{ \varphi(X_{t,h}, A_{t,h}) \varphi(X_{t,h}, A_{t,h})\transpose  }   \theta_{t,h}
\\
& =\EEt{\widehat{\Sigma}^{+}_{t,h} } \Sigma_{t,h} \theta_{t,h} 
= \theta_{t,h} - (I - \beta \Sigma_{t,h})^M\theta_{t,h}.
\end{align*}

Thus, the bias is bounded as
\[ 
\left|\EEt{\varphi(X_{t,h}, a)\transpose (I - \beta \Sigma_{t,h})^M \theta_{t,h} }\right| \le \twonorm{\varphi(X_{t,h}, a)}\cdot\twonorm{\theta_{t,h}}\opnorm{(I - \beta \Sigma_{t,h})^M}.
\]
In order to bound the last factor above, observe that $\Sigma_{t,h} \succcurlyeq \gamma \Sigma_h$ due to the uniform exploration used in the first layer by \linexprl, which implies that
\[
\opnorm{(I - \beta \Sigma_{t,h})^M}\le \pa{1 - \gamma\beta \lambdamin}^M
\le \exp\pa{ -\gamma \beta \lambdamin M},
\]
where the second inequality uses $1-z\le e^{-z}$ that holds for all $z$. This concludes the proof.
\jmlrQED

\subsection{The proof of Lemma~\ref{lem:propagation}}
The proof consists of two main components: proving that the conditional relative entropy between $u_{t}$ and $\hu_{t}$ 
can be bounded in terms of the optimization error $\varepsilon$, and then using this quantity to bound the 
total variation distance between these occupancy measures. For ease of readability, we state these results as separate 
lemmas.

We will first need the following statement:
\begin{lemma}\label{lem:projection}
$\DC(\hu_{t}\|u_{t})  \le \alpha \varepsilon$.
\end{lemma}
The proof follows along similar lines as the proof of Lemma~1 in~\citet{qreps}. To preserve clarity, we delegate its 
proof to Appendix~\ref{app:projection} below. The second lemma lemma bounds the relative entropy 
between two occupancy measures in terms of their \emph{conditional} relative entropies:
\begin{lemma}
 For any two occupancy measures $u$ and $u'$ and any $h$, we have
 \[
  D\pa{u_h\|u'_h} \le \sum_{k=1}^h \DC(u_k\|u_k').
 \]
\end{lemma}
\begin{proof}
 The proof follows from exploiting some basic properties of the relative entropy. Specifically, the result follows from 
the following chain of inequalities:
\begin{align*}
 D(u_h\|u_h') &= D(E\transpose u_h\|E\transpose u_h') + \DC(u_h\|u_h')
 \\
 &\qquad\mbox{(by the chain rule of the relative entropy)}
 \\
 &= D(P\transpose u_{h-1}\|P\transpose u_{h-1}') + \DC(u_h\|u_h')
 \\
 &\qquad\mbox{(by the fact that $u$ and $u'$ are valid occupancy measures)}
 \\
 &\le D(u_{h-1}\|u_{h-1}') + \DC(u_h\|u_h')
 \\
 &\qquad\mbox{(by the data processing inequality)}
 \\
 &\le \dots \le \sum_{k=1}^h \DC(u_k\|u_k'),
\end{align*}
where the last step follows from iterating the same argument for all layers.
\end{proof}
Putting the above two lemmas together and using Pinsker's inequality, we obtain
\[
 \onenorm{\hu_{t,h} - u_{t,h}} \le \sqrt{2 D\bpa{\hu_{t,h}\big\|u_{t,h}}} \le \sqrt{2 \sum_{k=1}^h 
\DC\bpa{\hu_{t,k}\big\|u_{t,k}}} \le \sqrt{2 
\DC\bpa{\hu_{t}\big\|u_{t}}} \le \sqrt{2 \alpha \varepsilon},
\]
concluding the proof of Lemma~\ref{lem:propagation}.
\jmlrQED

\subsection{The proof of Lemma~\ref{lem:projection}}\label{app:projection}
For the proof, let us introduce the notation $\tmu_{t,h}$ with
\[
 \tmu_{t,h}(x,a) = \frac{\mu_{0,h}(x,a) e^{\eta \Delta_{t, Z_t}(x,a)}}{\sum_{(x',a')\in(\ssp_h \times \A)} 
\mu_{0,h}(x,a) e^{\eta 
\Delta_{t, Z_t}(x,a)}}.
\]
and also $\GG_{t,h}(Z) = \frac{1}{\eta} \log\pa{\sum_{x\in \mathcal{X}_h,a\in A(x)} \mu_0(x,a) e^{\eta 
\Delta_{t,Z}(x,a)}}$ and $Z_t^* = \argmin_Z \GG_t(Z)$.
Then, observe that
\begin{align*}
 D(\hmu_{t,h}\|\tmu_{t,h}) =& \sum_{x,a \in \ssp_h\times\actions} \hmu_{t,h}(x,a) \log 
\frac{\hmu_{t,h}(x,a)}{\tmu_{t,h}(x,a)} 
 \\
 =& \eta \biprod{\hmu_{t,h}}{\Delta_{t,Z^*_t} - \GG_{t,h}(Z_t^*)\bone - \Delta_{t,Z_t} + \GG_{t,h}(Z_t)\bone}
 \\
 =& \eta \biprod{\hmu_{t,h}}{  P_h V_{Z_t^*} - Q_{Z_t^*}  - P_h V_{Z_t}   + Q_{Z_t}} + \eta \pa{\GG_{t,h}(Z_t^*) - 
\GG_{t,h}(Z_t)}
 \\
 =&  \eta \sum_{(x,a)\in (\ssp_h\times\A)} \sum_{x' \in \ssp_{h+1}} \hmu_{t,h}(x,a)P(x'|x,a)(V_{Z^*_t}(x') - 
V_{Z_t}(x') ) \\
 & + \eta \pa{\GG_{t,h}(Z_t^*) - 
\GG_{t,h}(Z_t)} + \eta \sum_{(x,a)\in (\ssp_h\times\A)} \hmu_{t,h}(x,a) \varphi(x,a)\transpose (Z_{t,h} - Z^*_{t,h} 
)\varphi(x,a) 
 \\
 =&  \eta \sum_{(x',a')\in (\ssp_{h+1}\times\A)}   \wh u_{t,h+1}(x',a')(V_{Z^*_t}(x') - V_{Z_t}(x') ) + \eta 
(\GG_{t,h}(Z_t) - \GG_{t,h}(Z_t^*)).\\
 & + \eta \sum_{(x,a)\in (\ssp_h\times\A)} \wh u_{t,h}(x,a) \varphi(x,a)\transpose (Z_{t,h} - Z^*_{t,h} )\varphi(x,a).
\end{align*}
Here, the last equality follows from the fact that $(\hmu_t,\hu_t)$ satisfy the constraints of the optimization 
problem~\eqref{opt}.
On the other hand, we have 
\begin{align*}
 \DC(\wh u_{t,h}\|u_{t,h}) =& \sum_{(x,a)\in (\ssp_h\times\A)} \wh  u_{t,h}(x,a) \log \frac{\wh 
\pi_{t,h}(a|x)}{\pi_{t,h}(a|x)} \\
 =& \alpha \sum_{(x,a)\in (\ssp_h\times\A)} \wh u_{t,h}(x,a) \sum_{x' \in \ssp_{h+1}} P(x'|x,a)(V_{Z^*_t}(x') - 
V_{Z_t}(x') ) \\
 &+ \alpha\sum_{(x,a)\in (\ssp_h\times\A)} \wh u_{t,h}(x,a) \varphi(x,a)\transpose (Z_{t,h} - Z^*_{t,h} )\varphi(x,a)\\
 =& \alpha \sum_{(x',a')\in (\ssp_{h+1}\times\A)} \wh u_{t,h+1}(x,a) (V_{Z^*_t}(x') - V_{Z_t}(x') ) \\
 &+ \alpha\sum_{(x,a)\in (\ssp_h\times\A)} \wh u_{t,h}(x,a) \varphi(x,a)\transpose (Z_{t,h} - Z^*_{t,h} )\varphi(x,a),
\end{align*}
where the last equality follows from the fact that $\hu_t$ is a valid occupancy measure, as shown in 
Equation~\eqref{eupu}.
Putting the two equalities together, we get
\[
 \frac{D(\wh \mu_{t,h}\|\mu_{t,h})}{\eta} - \frac{\DC(\wh u_{t,h}\|u_{t,h})}{\alpha} = \GG_{t,h}(Z_t) - 
\GG_{t,h}(Z_t^*).
\]
Then, summing up over all $h$ gives
\[
 \frac{D(\wh \mu_{t}\|\mu_{t})}{\eta} - \frac{\DC(\wh u_{t}\|u_{t})}{\alpha} = \sum_{h=1}^H \pa{\GG_{t,h}(Z_t) - 
\GG_{t,h}(Z_t^*)} = \GG_{t}(Z_t) - \GG_{t}(Z_t^*) \le \varepsilon.
\]
Reordering gives the result.
\jmlrQED

\subsection{The proof of Lemma~\ref{lem:bounded}}\label{app:bounded}
	The claim is proven by the following straightforward calculation:
	\begin{align*}
	\eta\cdot \big| \biprod{\varphi(X_{t,h}, a)}{\htheta_{t}}\big| 
	&= \eta \cdot \big| \varphi(X_{t,h}, a)\transpose  \widehat{\Sigma}^{+}_{t,h} \varphi(X_{t,h}, a) \iprod{\varphi(X_{t,h}, a)}{\theta_{t}}\big| 
	\\
	& \le \eta  \big| \varphi(X_{t,h}, a)\transpose  \widehat{\Sigma}^{+}_{t,h} \varphi(X_{t,h}, a)  \big|
	\le \eta  \sigma^2 \opnorm{\widehat{\Sigma}^{+}_{t,h}} \\
	& \le 
	\eta  \sigma^2 \beta\pa{1+  \sum_{k=1}^M \opnorm{C_{k,h}}} \le \eta  (M+1) /2,
	\end{align*}
	where we used the fact that our choice of $\beta$ ensures $\opnorm{C_{k,h}} = \opnorm{\prod_{j=0}^k(I-\beta B_{j,h})} \le 1$.  
\jmlrQED

%% file: appendixD.tex
\section{Fast Matrix Geometric Resampling}
\label{appendixD}

The na\"ive implementation of the MGR procedure presented in the main text requires $O(MKHd + MHd^2)$ 
time due to the matrix-matrix multiplications involved. In this section we explain how to compute 
$\htheta_{t}$ in $O(MKHd)$ time, exploiting the fact that the 
matrices $\wh{\Sigma}_{t,h}$ never actually need to be computed, since the algorithm only works 
with products of the form $\wh{\Sigma}_{t,h}\varphi(X_{t,h}, A_{t,h})$ for  vectors $X_{t,h}$, $h\in [H]$. This motivates 
the following procedure:

\makebox[\textwidth][c]{
	\fbox{
		\begin{minipage}{.85\textwidth}
			\textbf{Fast Matrix Geometric Resampling}
			\vspace{.1cm}
			\hrule
			\vspace{.1cm}
            \textbf{Input:} simulator of transition function $P$, policy $\pi_t$\\
			\textbf{Initialization:} Compute $Y_{0,h} =  \varphi(x_h)$ for all $h\in[H]$.\\
			\textbf{For $k = 1, \dots, M$, repeat}:
			\begin{enumerate}
				\vspace{-1mm}
				\item Generate a path $U(i) = \{ (X_1(i), A_1(i) ), \dots,  (X_H(i), A_H(i) ) \}$, \\following the policy $\pi_t$ in the simulator of $P$,
				\item \textbf{For $h = 1, \dots, H$, repeat}:
				\begin{enumerate}
				\item if $A_h(k)=a_h$, set $Y_{k,h} = Y_{k-1,h} - \beta\iprod{Y_{k-1,h}}{\varphi(X_h(k), A_h(k))}\varphi(X_h(k), A_h(k)),$ 
				\item otherwise, set $Y_{k,h} = Y_{k-1,h}$.
				\end{enumerate}
			\end{enumerate}
			\vspace{-1mm}
			\textbf{Return} $q_{t,h} = \beta Y_{0,h} + \beta \sum_{k=1}^M Y_{k,h}$ for all $h \in [H]$.
		\end{minipage}
	}
}
\vspace{0.3cm}

It is easy to see from the above procedure that each iteration $k$ can be computed using 
$(K+1)Hd$ vector-vector multiplications: sampling each action $A_h(k)$ takes $Kd$ time due to having to 
compute the products $\biprod{\varphi(X_h(k))}{\sum_{s=1}^{t-1}\htheta_{s, a, h}}$ for each action $a$, and updating $Y_{k,h}$ can 
be done by computing the product $\iprod{Y_{k-1,h}}{\varphi(X_h(k))}$. Overall, this results in a total runtime 
of order $MKHd$ as promised above.

%% file: appendixF.tex
\section{Implementation by optimizing the empirical loss}\label{app:implementation}
This section outlines a possible implementation of the policy update steps based on approximate minimization of an 
empirical counterpart of the loss function $\GG_t$. To this end, we define 
\[
\GG_{t,h}(Z) = \frac 1\eta \log \pa{\sum_{x,a} \mu_0(x,a) e^{\eta \Delta_{Z,t,h}(x,a)}}
\]
and its empirical counterpart that replaces the expectation by an empirical mean over state-action pairs sampled from 
$\mu_0$. Concretely, for all $h$, we let $\pa{X_h(i),A_h(i)}_{i=1}^N$ be $N$ independent samples from $\mu_0$ that can 
be obtained by running policy $\pi_0$ in the transition model $P$. 
Using these samples, we define
\begin{equation}\label{eq:loss_estimator}
 \wh{\GG}_{t,h}(Z) = \frac 1\eta \log \pa{\sum_{n=1}^N e^{\eta \Delta_{Z,t,h}(X_{h}(i),A_{h}(i))}}.
\end{equation}
This objective function has several desirable properties: it is convex in $Z$, has bounded gradients, and is 
$(\alpha + \eta)$-smooth. Furthermore, its gradients can be evaluated efficiently in $\mathcal{O}(N)$ time, given that 
we can efficiently evaluate expectations of the form $\sum_{x'} P(x'|x,a) V(x')$. As a result, it can be optimized up 
to arbitrary precision $\varepsilon$ in time polynomial in $1/\varepsilon$ and $N$.

The downside of this estimator is that it is potentially biased. Nevertheless, as the following lemma shows, it is 
well-concentrated around the true objective function, under some reasonable conditions:
\begin{lemma}
Fix $Z$ and suppose that $|\Delta_Z(x,a)|\le B$ for all $x,a$.
Then, with probability at least $1-\delta$, the following holds:
\[
 \left|\wh\GG_{t,h}(Z) - \GG_{t,h}(Z) \right| \le 56 \sqrt{\frac{\log(1/\delta)}{N}}.
\]
\end{lemma}
This statement is a variant of Theorem~1 from \citet{qreps}, with the key difference being that being able to exactly 
calculate expectations with respect to $P(\cdot|x,a)$ enables us to prove a tighter bound.
\begin{proof}
Let us start by defining the shorthand notations $\wh{S}_i = \Delta_{Z,t}(X_{h}(i),A_{h}(i))$ and $\bW = \frac 1N 
\sum_{i=1}^N e^{\eta S_i}$. Furthermore, we define the function
\[
 f(s_1,s_2,\dots,s_N) = \frac 1N \sum_{i=1}^N e^{\eta s_i}
\]
and notice that it satisfies the bounded-differences property
\[
 f(s_1,s_2,\dots,s_i,\dots,s_N) - f(s_1,s_2,\dots,s_i',\dots,s_N) = \frac 1N \pa{e^{\eta s_i} - e^{\eta 
s_i'}} \le \frac {\eta e^{2\eta B}}N.
\]
Here, the last step follows from Taylor's theorem that implies that there exists a $\chi\in(0,1)$ such that
\[
 e^{\eta s_i'} = e^{\eta s_i} + \eta e^{\eta \chi\pa{s_i' - s_i}}
\]
holds, so that $e^{\eta s_i'} - e^{\eta s_i} = \eta e^{\eta \chi\pa{s_i' - s_i}} \le \eta e^{2\eta B}$, where we used 
the assumption that $|s_i - s_i'| \le 2B$ in the last step. Notice that our assumption $\eta B\le 1$ further implies 
that $e^{2\eta B}\le e^2$. Thus, also noticing that $W = f(S_1,\dots,S_N)$, we can apply McDiarmid's 
inequality that to show that the following holds with probability at least $1-\delta'$:
\begin{equation}\label{eq:Wdiff}
 |W - \EE{W}| \le \eta e^2 \sqrt{\frac{\log(2/\delta')}{N}}.
\end{equation}

Thus, we can write
\begin{align*}
\wh{\GG}_{t,h}(\theta) - \GG_{t,h}(\theta) 
&= \frac 1\eta \log\pa{W} - \frac 1\eta \log\pa{\EE{\bW}} = \frac 1\eta 
\log\pa{\frac{W}{\EE{\bW}}}
\\
& = \frac 1\eta \log\pa{1 + \frac{W - \EE{\bW}}{\EE{\bW}}} \le \frac{W - \EE{\bW}}{\eta \EE{\bW}} \le e^4 
\sqrt{\frac{\log(2/\delta')}{N}},
\end{align*}
where the last line follows from the inequality $\log(1+u) \le u$ that holds for $u>-1$ and our assumption on $\eta$ 
that implies $\bW\ge e^{-2}$. Similarly, we can show
\begin{align*}
 \GG_{t,h}(\theta) - \wh{\GG}_{t,h}(\theta) &= \frac 1\eta \log\pa{1 + \frac{\EE{\bW} - W}{W}}
 \le 
 \frac{\EE{W} - W}{\eta W} \le 
e^4 \sqrt{\frac{\log(2/\delta')}{N}}.
\end{align*}
This concludes the proof.
\end{proof}